\documentclass[runningheads]{llncs}
\usepackage[T1]{fontenc}

\usepackage[square,numbers]{natbib}
\usepackage{graphicx}
\usepackage{booktabs}       
\usepackage{amsfonts}       
\usepackage{nicefrac}       
\usepackage{microtype}      
\usepackage{mathtools}
\usepackage{amsmath}
\usepackage{diagbox} 
\usepackage{siunitx} 
\usepackage{multirow}
\usepackage{xcolor}         
\usepackage{hyperref}
\hypersetup{
    colorlinks,
    linkcolor={red!50!black},
    citecolor={blue!50!black},
    urlcolor={blue!80!black}
}
\usepackage[misc]{ifsym}

\usepackage{algorithm2e}
\usepackage{xcolor}
\usepackage{wrapfig}

\SetAlCapSty{xAlCapSty}

\SetCommentSty{xCommentSty}
\SetNlSty{mynlfont}{}{} 
\LinesNumbered
\SetSideCommentRight
\DontPrintSemicolon
\RestyleAlgo{algoruled}
\SetKwInput{KwData}{Input}
\SetKwInput{KwResult}{Output}

\begin{document}
\title{Fooling Partial Dependence via Data Poisoning}
\titlerunning{Fooling Partial Dependence via Data Poisoning}
\toctitle{Fooling Partial Dependence via Data Poisoning}
\author{Hubert Baniecki \Letter \and Wojciech Kretowicz \and Przemyslaw Biecek}
\authorrunning{H. Baniecki et al.}
\tocauthor{Hubert~Baniecki, Wojciech~Kretowicz, and Przemyslaw~Biecek}
%
\institute{Warsaw University of Technology, Warsaw, Poland\\
\email{\{hubert.baniecki.stud,wojciech.kretowicz.stud,przemyslaw.biecek\}@pw.edu.pl}}
\maketitle              
\begin{abstract}
Many methods have been developed to understand complex predictive models and high expectations are placed on post-hoc model explainability. It turns out that such explanations are not robust nor trustworthy, and they can be fooled. This paper presents techniques for attacking Partial Dependence (plots, profiles, PDP), which are among the most popular methods of explaining any predictive model trained on tabular data. We showcase that PD can be manipulated in an adversarial manner, which is alarming, especially in financial or medical applications where auditability became a must-have trait supporting black-box machine learning. The fooling is performed via poisoning the data to bend and shift explanations in the desired direction using genetic and gradient algorithms. We believe this to be the first work using a~genetic algorithm for manipulating explanations, which is transferable as it generalizes both ways: in a~model-agnostic and an~explanation-agnostic manner.

\end{abstract}
\section{Introduction}\label{sec:introduction}

Although supervised machine learning became state-of-the-art solutions to many predictive problems, there is an emerging discussion on the underspecification of such methods which exhibits differing model behaviour in training and practical setting \citep{underspecification}. This is especially crucial when proper accountability for the systems supporting decisions is required by the domain \citep{Lipton2018, Miller2019, Rudin2019}. Living with black-boxes, several explainability methods were presented to help us understand models' behaviour \citep{ale, friedman-gbm-pdp, goldstein-ice, shap, lime}, many are designed specifically for deep neural networks \citep{lrp, deeplearning, input-gradient, integrated-gradient}. Explanations are widely used in practice through their (often estimation-based) implementations available to machine learning practitioners in various software contributions \citep{nn-software, dalex-python, dalex}. 

Nowadays, robustness and certainty become crucial when using explanations in the data science practice to understand black-box machine learning models; thus, facilitate rationale explanation, knowledge discovery and responsible decision-making \citep{arrieta-responsible-ai, responsible-ml}. Notably, several studies evaluate explanations \citep{adebayo-sanity-checks-saliency-maps, eval-4, eval-3, eval-1, explanation-quality-vs-model-accuracy, eval-2} showcasing their various flaws from which we perceive an existing robustness gap; in critical domains, one can call it a \textit{security breach}. Apart from promoting wrong explanations, this phenomenon can be exploited to use adversarial attacks on model explanations to achieve manipulated results. In regulated areas, these types of attacks may be carried out to deceive an auditor. 

\begin{figure}[!t]	
\centering
\includegraphics[width=0.95\columnwidth]{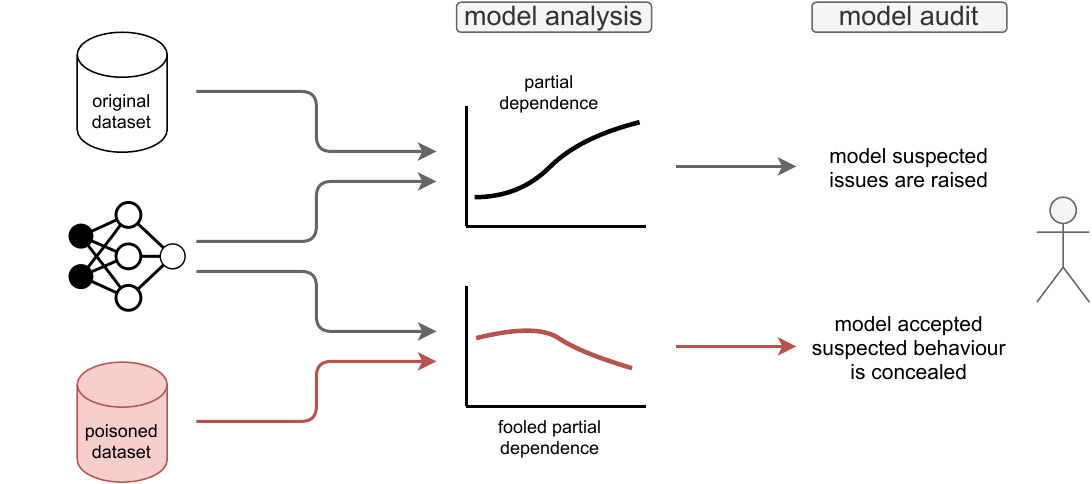}
\caption[Framework for fooling model explanations via data poisoning.]{Framework for fooling model explanations via data poisoning. The red color indicates the adversarial route, a~potential security breach, which an attacker may use to manipulate the explanation. Researchers could use this method to provide a~misleading rationale for a given phenomenon, while auditors may purposely conceal the suspected, e.g. biased or irresponsible, reasoning of a black-box model.}
\label{fig:framework}
\end{figure}  

Figure \ref{fig:framework} illustrates a process in which the developer aims to conceal the undesired behaviour of the model by supplying a poisoned dataset for model audit. Not every explanation is equally good---just as models require proper performance validation, we need similar assessments of explainability methods. In this paper, we evaluate the robustness of Partial Dependence (PD) \citep{friedman-gbm-pdp}, moreover highlight the possibility of adversarial manipulation of PD (see Figures \ref{fig:heart-age} \& \ref{fig:heart-sex} in the latter part of the paper). We summarize the contributions as follows: 

\paragraph{(1)} We introduce a novel concept of using a \emph{genetic algorithm} for manipulating model explanations. This allows for a convenient generalization of the attacks in a model-agnostic and explanation-agnostic manner, which is not the case for most of the related work. Moreover, we use a gradient algorithm to perform fooling via data poisoning efficiently for neural networks. 

\paragraph{(2)} We explicitly target PD to highlight the potential of their adversarial manipulation, which was not done before. Our method provides a sanity check for the future use of PD by responsible machine learning practitioners. Evaluation of the constructed methods in extensive experiments shows that model complexity significantly affects the magnitude of the possible explanation manipulation.

\section{Related Work}\label{sec:related-work}

In the literature, there is a considerable amount of attacks on model explanations specific to deep neural networks \citep{dombrowski-manipulated-explanations, ghorbani-fragile-nn-interpretability, heo-fooling-nn-manipulation, kindermans-unreliable-saliency-maps, zhang-interpretable-dl-under-fire}. At their core, they provide various algorithms for fooling neural network interpretability and explainability, mainly of image-based predictions. Such explanations are commonly presented through saliency maps \citep{saliency1}, where each model input is given its attribution to the prediction \citep{lrp, gradcam, input-gradient, integrated-gradient}. Although explanations can be used to improve the adversarial robustness of machine learning models \citep{adversarial-robustness}, we target explanations instead. When considering an explanation as a function of model and data, there is a~possibility to change one of these variables to achieve a different result~\citep{zhao2019causal}. Heo et al. \citep{heo-fooling-nn-manipulation} and Dimanov et al. \citep{attack-model-fairness} propose fine-tuning a neural network to undermine its explainability capabilities and conceal model unfairness. The assumption is to alter the model's parameters without a drop in performance, which can be achieved with an objective function minimizing the distance between explanations and an arbitrarily set target. Note that \citep{attack-model-fairness} indirectly change partial dependence while changing the model. Aivodji et al. \citep{fairwashing} propose creating a~surrogate model aiming to approximate the unfair black-box model and explain its predictions in a fair manner, e.g. with relative variable dependence. 

Alternate idea is to fool fairness and explainability via data change since its (background) distribution greatly affects interpretation results~\citep{janzing2020feature, kindermans-unreliable-saliency-maps}. Solans et al. \citep{poisoning-fairness} and Fukuchi et al. \citep{faking-fairness} investigate concealing unfairness via data change by using gradient methods. Dombrowski et al. \citep{dombrowski-manipulated-explanations} propose an algorithm for saliency explanation manipulation using gradient-based data perturbations. In contrast, we introduce a genetic algorithm and focus on other machine learning predictive models trained on tabular data. Slack et al. \citep{slack-fooling-lime-shap} contributed adversarial attacks on post-hoc, model-agnostic explainability methods for \textit{local-level} understanding; namely LIME \citep{lime} and SHAP \citep{shap}. The proposed framework provides a way to construct a biased classifier with safe explanations of the model's predictions.

Since we focus on \emph{global-level} explanations, instead, the results will modify a view of overall model behaviour, not specific to a~single data point or image. Lakkaraju and Bastani \citep{misleading} conducted a~thought-provoking study on misleading effects of manipulated Model Understanding through
Subspace Explanations (MUSE) \citep{muse}, which provide arguments for why such research becomes crucial to achieve responsibility in machine learning use. Further, the robustness of neural networks \citep{nn-defense-1, nn-defense-2} and counterfactual explanations \citep{slack2021counterfactual} have became important, as one wants to trust black-box models and extend their use to sensitive tasks. Our experiments further extend to global explanations the indication of Jia et al. \citep{explanation-quality-vs-model-accuracy} that there is a correlation between model complexity and explanation quality.

\section{Partial Dependence}

In this paper, we target one of the most popular explainability methods for tabular data, which at its core presents the expected value of the model's predictions as a function of a selected variable. Partial Dependence, formerly introduced as \emph{plots} by Friedman \citep{friedman-gbm-pdp}, show the expected value fixed over the marginal joint distribution of other variables. These values can be relatively easily estimated and are widely incorporated into various tools for model explainability \citep{modelStudio, dalex-python, dalex, pdp, iml}. The theoretical explanation has its practical estimator used to compute the results, later visualized as a line plot showing the expected prediction for a given variable; also called \emph{profiles} \citep{biecek-ema}. PD for model $f$~and variable $c$ in a~random vector $\mathcal X$ is defined as $\mathcal{PD}_c(\mathcal X,z) \coloneqq E_{\mathcal X_{-c}}\left[f(\mathcal X^{c|=z})\right]$,
where $\mathcal X^{c|=z}$ stands for random vector $\mathcal X$, where $c$-th variable is replaced by value $z$. By $\mathcal X_{-c}$, we denote distribution of random vector $\mathcal X$ where $c$-th variable is set to a constant. We defined PD in point $z$~as the expected value of model $f$~given the $c$-th variable is set to $z$. The standard estimator of this value for data $X$ is given by the following formula $\widehat{\mathcal{PD}_c}(X,z) \coloneqq \frac{1}{N}\sum_{i=1}^N f\left(X_i^{c|=z}\right),$
where $X_i$ is the $i$-th row of the matrix $X$ and the previously mentioned symbols are used accordingly. To simplify the notation, we will use $\mathcal{PD}$, and omit $z$ and $c$~where context is clear.

\section{Fooling Partial Dependence via Data Poisoning}\label{sec:fooling}

Many explanations treat the dataset $X$ as fixed; however, this is precisely a \textit{single point of failure} on which we aim to conduct the attack. In what follows, we examine $\mathcal{PD}$ behaviour by looking at it as a function whose argument is an entire dataset. For example, if the dataset has $N$ instances and $P$ variables, then $\mathcal{PD}$ is treated as a function over $N\times P$ dimensions. Moreover, because of the complexity of black-box models, $\mathcal{PD}$ becomes an extremely high-dimensional space where variable interactions cause unpredictable behaviour. Explanations are computed using their estimators where a significant simplification may occur; thus, a slight shift of the dataset used to calculate $\mathcal{PD}$ may lead to unintended~results (for example, see \citep{hooker2007} and the references given there).

We aim to change the underlying dataset used to produce the explanation in a~way to achieve the desired change in $\mathcal{PD}$. Figure \ref{fig:framework} demonstrates the main threat of an adversarial attack on model explanation using data poisoning, which is \emph{concealing the suspected behaviour of black-box models}. The critical assumption is that an adversary has a possibility to modify the dataset arbitrarily, e.g. in healthcare and financial audit, or research review. Even if this would not be the case, in modern machine learning, wherein practice dozens of variables are used to train complex models, such data shifts might be only a~minor change that a person looking at a~dataset or distribution will not be able to identify. 

We approach fooling $\mathcal{PD}$ as an optimization problem for given criteria of attack efficiency, later called the attack loss. It originates from \citep{dombrowski-manipulated-explanations}, where a similar loss function for manipulation of local-level model explanations for an image-based predictive task was introduced. This work introduces the attack loss that aims to change the output of a global-level explanation via data poisoning instead. The explanation's weaknesses concerning data distribution and causal inference are exploited using two ways of optimizing the loss: 

\begin{itemize}
    \item \textbf{Genetic-based}\footnote{For convenience, we shorten the \textit{algorithm based on the genetic algorithm} phrase to \textit{genetic-based algorithm}.} algorithm that does not make any assumption about the model's structure -- the black-box path from data inputs to the output prediction; thus, is model-agnostic. Further, we posit that for a vast number of explanations, clearly post-hoc global-level, the algorithm does not make assumption about their structure either; thus, becomes \textit{explanation-agnostic}.
    \item \textbf{Gradient-based} algorithm that is specifically designed for models with differentiable outputs, e.g.~neural networks \citep{attack-model-fairness, dombrowski-manipulated-explanations}.
\end{itemize}
We discuss and evaluate two possible fooling strategies:
\begin{itemize}
    \item \textbf{Targeted attack} changes the dataset to achieve the closest explanation result to the predefined desired function \citep{dombrowski-manipulated-explanations, heo-fooling-nn-manipulation}.
    \item \textbf{Robustness check} aims to achieve the most distant model explanation from the original one by changing the dataset, which refers to the sanity check \citep{adebayo-sanity-checks-saliency-maps}.
\end{itemize}
For practical reasons, we define the distance between the two calculated $\mathcal{PD}$ vectors as $\|x - y\| \coloneqq \frac{1}{I}\sum_{i=1}^{I} (x_i - y_i)^{2}$, yet other distance measures may be used to extend the method.

\subsection{Attack Loss}

The intuition behind the attacks is to find a modified dataset that minimizes the attack loss. A changed dataset denoted as $X\in \mathbb{R}^{N \times P}$ is an argument of that function; hence, an optimal $X$ is a~result of the attack. Let $Z \subset \mathbb{R}$ be the set of points used to calculate the explanation. Let $T:Z\rightarrow\mathbb{R}$ be the target explanation; we write just $T$ to denote a vector over whole $Z$. Let $g_c^Z:\mathbb{R}^{N \times P}\rightarrow \mathbb{R}^{|Z|}$ be the actual explanation calculated for points in $Z$; we write $g_c$ for simplicity. Finally, let $X'\in \mathbb{R}^{N \times P}$ be the original (constant) dataset. We define the attack loss as $\mathcal{L}(X) \coloneqq \mathcal{L}^{g,\;s}(X)$, where $g$~is the explanation to be fooled, and an objective is minimized depending on the strategy of the attack, denoted as $s$. The aim is to minimize $\mathcal{L}$ with respect to the dataset $X$ used to calculate the explanation. We never change values of the explained variable $c$ in the dataset.

In the \textbf{targeted attack}, we aim to \textbf{minimize} the distance between the target model behaviour $T$ and the result of model explanation calculated on the changed dataset. We denote this strategy by $t$ and define  $\mathcal{L}^{g,\;t}(X) = \|g_c(X) - T\|.$ Since we focus on a specific model-agnostic explanation, we substitute $\mathcal{PD}$ in place of $g$ to obtain $\mathcal{L}^{\mathcal{PD},\;t}(X) = \|\mathcal{PD}_c(X) - T\|.$ This substitution can be generalized for various global-level model explanations, which rely on using a part of the dataset for computation. 

In the \textbf{robustness check}, we aim to \textbf{maximize} the distance between the result of model explanation calculated on the original dataset $g_c(X')$, and the changed one; thus, minus sign is required. We denote this strategy by $r$ and define $\mathcal{L}^{g,\;r}(X) = -\|g_c(X) - g_c(X')\|.$ Accordingly, we substitute $\mathcal{PD}$ in place of $g$ to obtain $\mathcal{L}^{\mathcal{PD},\;r}(X) = -\|\mathcal{PD}_c(X) - \mathcal{PD}_c(X')\|.$ Note that $\mathcal{L}^{g,\;s}$ may vary depending on the explanation used, specifically for $\mathcal{PD}$ it is useful to centre the explanation before calculating the distances, which is the default behaviour in our implementation: $\mathcal{L}^{\mathcal{\overline{PD}},\;r}(X) = -\|\mathcal{\overline{PD}}_c(X) - \mathcal{\overline{PD}}_c(X') \|,$ where $\mathcal{\overline{PD}}_c \coloneqq \mathcal{PD}_c(X) - \frac{1}{|Z|}\sum_{z\in Z} \mathcal{PD}_c(X,z).$ We consider the second approach of comparing explanations using centred $\mathcal{\overline{PD}}$, as it forces changes in the shape of the explanation, instead of promoting to shift the profile vertically while the shape changes insignificantly. 

\subsection{Genetic-based Algorithm}\label{subsec:genetic-method}

We introduce a novel strategy for fooling explanations based on the genetic algorithm as it is a simple yet powerful method for real parameter optimization~\citep{real-ga}. We do not encode genes conventionally but deliberately use this term to distinguish from other types of evolutionary algorithms \citep{5-evolutionary-algorithms}. The method will be invariant to the model's definition and the considered explanations; thus, it becomes model-agnostic and explanation-agnostic. These traits are crucial when working with black-box machine learning as versatile solutions are convenient. 

Fooling $\mathcal{PD}$ in both strategies include a similar genetic algorithm. The main idea is to define an individual as an instance of the dataset, iteratively perturb its values to achieve the desired explanation target, or perform the robustness check to observe the change. These individuals are initialized with a value of the original dataset $X'$ to form a population. Subsequently, the initialization ends with mutating the individuals using a higher-than-default variance of perturbations. Then, in each iteration, they are randomly crossed, mutated, evaluated with the attack loss, and selected based on the loss values. \textbf{Crossover} swaps columns between individuals to produce new ones, which are then added to the population. The number of swapped columns can be randomized; also, the number of parents can be parameterized. \textbf{Mutation} adds Gaussian noise to the individuals using scaled standard deviations of the variables. It is possible to constrain the data change into the original range of variable values; also keep some variables unchanged. \textbf{Evaluation} calculates the loss for each individual, which requires to compute explanations for each dataset. \textbf{Selection} reduces the number of individuals using rank selection, and elitism to guarantee several best individuals to remain in the next population.

We considered the crossover through an exchange of rows between individuals, but it might drastically shift the datasets and move them apart. Additionally, a worthy mutation is to add or subtract whole numbers from the integer-encoded (categorical) variables. We further discuss the algorithm's details in the Supplementary material. The introduced attack is model-invariant because no derivatives are needed for optimization, which allows evaluating explanations of black-box models. While we found this method a sufficient generalization of our framework, there is a~possibility to perform a more efficient optimization assuming the prior knowledge concerning the structure of model and explanation.

\subsection{Gradient-based Algorithm}\label{subsec:gradient-method}

Gradient-based methods are state-of-the-art optimization approaches, especially in the domain of deep neural networks \citep{deeplearning}. This algorithm's main idea is to use gradient descent to optimize the attack loss, considering the differentiability of the model's output with respect to input data. Such assumption allows for a faster and more accurate convergence into a local minima using one of the stochastic optimizers; in our case, Adam \citep{adam}. Note that the differentiability assumption is with respect to input data, not with respect to the model's parameters. We shall derive the gradients $\nabla_{X_{-c}} \mathcal{L}^{g,\;s}$ for fooling explanations based on their estimators, not the theoretical definitions. This is because the input data is assumed to be a~random variable in a~theoretical definition of $\mathcal{PD}$, making it impossible to calculate a derivative over the input dataset. In practice, we do not derive our method directly from the definition as the estimator produces the explanation. 

Although we specifically consider the usage of neural networks because of their strong relation to differentiation, the algorithm's theoretical derivation does not require this type of model. For brevity, we derive the theoretical definitions of gradients $\nabla_{X_{-c}} \mathcal{L}^{\mathcal{PD},\;t}$, $\nabla_{X_{-c}} \mathcal{L}^{\mathcal{PD},\;r}$, and $\nabla_{X_{-c}} \mathcal{L}^{\mathcal{\overline{PD}},\;r}$ in the Supplementary material. Overall, the gradient-based algorithm is similar to the genetic-based algorithm in that we aim to iteratively change the dataset used to calculate the explanation. Nevertheless, its main assumption is that the model provides an interface for the differentiation of output with respect to the input, which is not the case for black-box models. 

\section{Experiments}\label{sec:experiments}

\paragraph{Setup.} We conduct experiments on two predictive tasks to evaluate the algorithms and conclude with illustrative scenario examples, which refer to the framework shown in Figure \ref{fig:framework}. The first dataset is a synthetic regression problem that refers to the Friedman's work \citep{friedman-gbm-pdp} where inputs $X$ are independent variables uniformly distributed on the interval $[0, 1]$, while the target $y$ is created according to the formula: $y(X) = 10 \sin(\pi \boldsymbol{\cdot} X_1 \boldsymbol{\cdot} X_2) + 20 (X_3 - 0.5)^{2} + 10 X_4 + 5  X_5.$ Only $5$ variables are actually used to compute $y$, while the remaining are independent of. We refer to this dataset as \texttt{friedman} and target explanations of the variable~$X_1$. The second dataset is a real classification task from UCI \citep{UCI}, which has $5$~continuous variables, $8$~categorical variables, and an evenly-distributed binary target. We refer to this dataset as \texttt{heart} and target explanations of the variable \texttt{age}. Additionally, we set the discrete variables as constant during the performed fooling because we mainly rely on incremental change in the values of continuous variables, and categorical variables are out of the scope of this work. 

\paragraph{Results.} Figure \ref{fig:experiments-gradient} present the main result of the paper, which is that PD can be manipulated. We use the gradient-based algorithm to change the explanations of feedforward neural networks via data poisoning. The targeted attack aims to arbitrarily change the monotonicity of PD, which is evident in both predictive tasks. The robustness check finds the most distant explanation from the original one. We perform the fooling 30 times for each subplot, and the Y-axis denotes the model's architecture: layers$\times$neurons. We observe that PD explanations are especially vulnerable in complex models. 

Next, we aim to evaluate the PD of various state-of-the-art machine learning models and their complexity levels; we denote: Linear Model (LM), Random Forest (RF), Gradient Boosting Machine (GBM), Decision Tree (DT), K-Nearest Neighbours (KNN), feedforward Neural Network (NN). The model-agnostic nature of the genetic-based algorithm allows this comparison as it might be theoretically and/or practically impossible to differentiate the model's output with respect to the input; thus, differentiate the explanations and loss.

\begin{figure}[!ht]
  \centering
  \includegraphics[width=0.8\columnwidth]{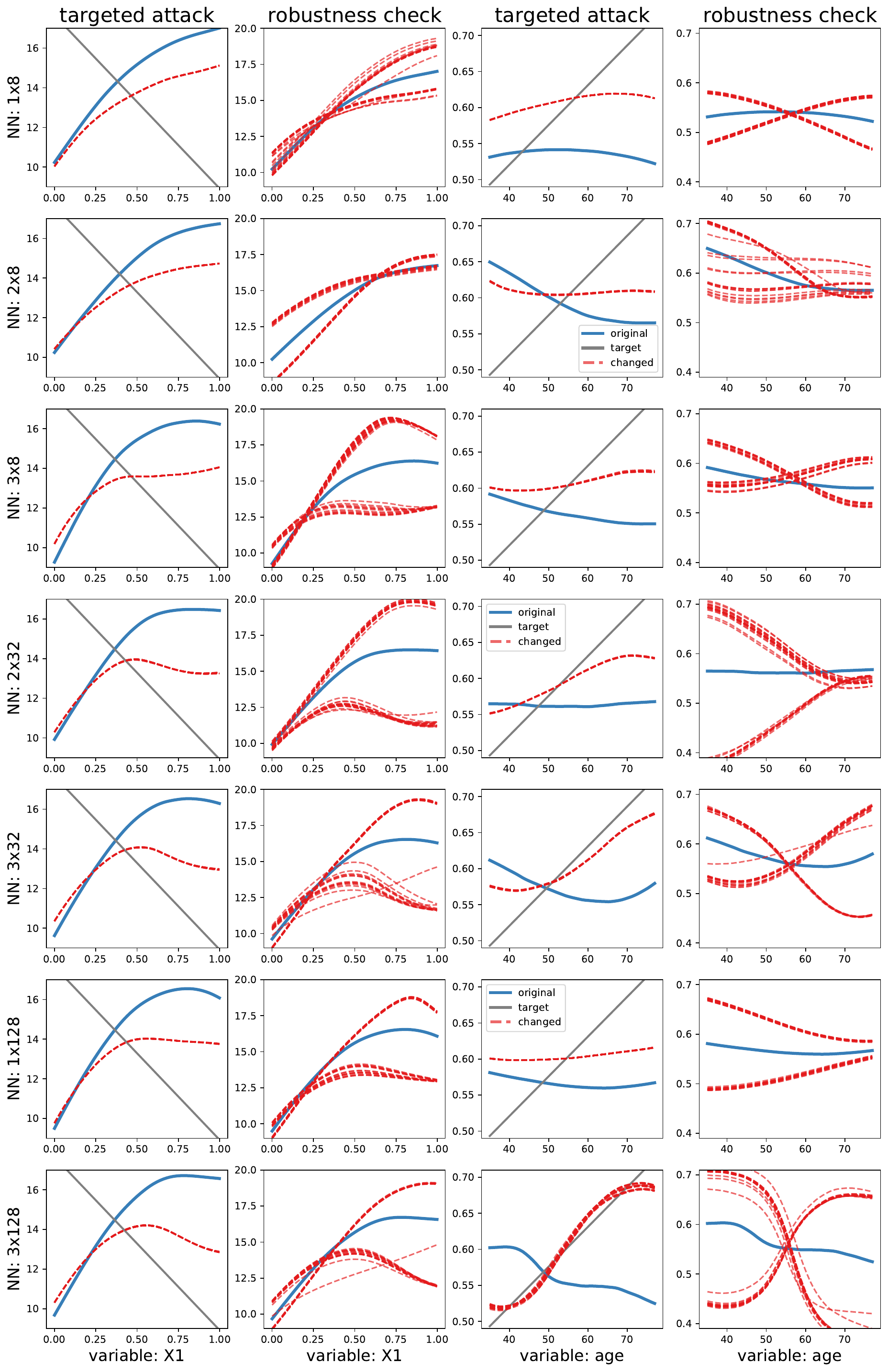}
  \caption{Fooling Partial Dependence of neural network models (rows) fitted to the \texttt{friedman} and \texttt{heart} datasets (columns). We performed multiple randomly initiated gradient-based fooling algorithms on the explanations of variables $X_1$ and \texttt{age} respectively. The blue line denotes the original explanation, the red lines are the fooled explanations, and in the targeted attack, the grey line denotes the desired target. We observe that the explanations' vulnerability greatly increases with model complexity. Interestingly, the algorithm seems to converge to two contrarary optima when no target is provided.}\label{fig:experiments-gradient}
\end{figure}

\clearpage

\begin{table}[ht]
  \setlength{\tabcolsep}{4pt}      
  \centering
  \caption{Attack loss values of the robustness checks for Partial Dependence of various machine learning models (\textbf{top}), and complexity levels of tree-ensembles (\textbf{bottom}). Each value corresponds to the scaled distance between the original explanation and the changed one. We perform the fooling 6 times and report the $\text{mean} \pm \text{sd}$. We observe that the explanations’ vulnerability increases with GBM complexity.}
  \begin{tabular}{cccccccc}
    \toprule
    \backslashbox{Task}{Model} & LM & RF & GBM & DT & KNN & NN & SVM \\
    \midrule
    \texttt{friedman}  & $0_{\pm 0}$ & $152_{\pm 76}$ & $127_{\pm 71}$ & $332_{\pm 172}$ & $164_{\pm 61}$ & $269_{\pm 189}$ & $576_{\pm 580}$\\ 
    \midrule
    \texttt{heart}  & $2_{\pm 3}$ & $20_{\pm 5}$ & $77_{\pm 28}$ & $798_{\pm 192}$ & $133_{\pm 21}$ & $501_{\pm 52}$ & $451_{\pm 25}$\\ 
    \bottomrule
  \end{tabular}\\
    \vspace{1em}
  \begin{tabular}{cccccccc}
    \toprule
    Task & \backslashbox{Model}{Trees} & 10 & 20 & 40 & 80 & 160 & 320 \\
    \midrule
    \multirow{2}*{\texttt{friedman}} & GBM  & $57_{\pm 12}$ & $114_{\pm 20}$ & $157_{\pm 37}$ & $176_{\pm 20}$ & $189_{\pm 8}$ & $210_{\pm 9}$\\
    & RF  & $233_{\pm 22}$ & $219_{\pm 25}$ & $219_{\pm 9}$ & $201_{\pm 23}$ & $216_{\pm 13}$ & $209_{\pm 15}$\\
    \midrule
    \multirow{2}*{\texttt{heart}} & GBM  & $1_{\pm 0}$ & $3_{\pm 1}$ & $29_{\pm 4}$ & $70_{\pm 24}$ & $152_{\pm 56}$ & $321_{\pm 95}$\\
    & RF  & $62_{\pm 7}$ & $55_{\pm 3}$ & $29_{\pm 9}$ & $21_{\pm 6}$ & $14_{\pm 5}$ & $13_{\pm 2}$\\
    \bottomrule
  \end{tabular}
  \label{tab:experimens-genetic}
\end{table}

Table \ref{tab:experimens-genetic} presents the results of robustness checks for Partial Dependence of various machine learning models and complexity levels. Each value corresponds to the distance between the original explanation and the changed one; multiplied by $10^3$ in \texttt{friedman} and $10^6$ in \texttt{heart} for clarity. We perform the checks 6 times and report the $\text{mean} \pm \text{standard deviation}$. Note that we cannot compare the values between tasks, as their magnitudes depend on the prediction range. We found the explanations of NN, SVM and deep DT the most vulnerable to the fooling methods (\textbf{top} Table). In contrast, RF seems to provide robust explanations; thus, we further investigate the relationship between the tree-models' complexity and the explanation stability (\textbf{bottom} Table) to conclude that an increasing complexity yields more vulnerable explanations, which is consistent with Figure \ref{fig:experiments-gradient}. We attribute the differences between the results for RF and GBM to the concept of bias-variance tradeoff. In some cases (\texttt{heart}, RF), explanations of too simple models become vulnerable too, since underfitted models may be as uncertain as overfitted ones.

\paragraph{Ablation Study.} We further discuss the additional results that may be of interest to gain a broader context of this work. Figure \ref{fig:add-experiments-centred} presents the distinction between the robustness check for centred Partial Dependence, which is the default algorithm, and the robustness check for \emph{not centred} PD. We use the gradient-based algorithm to change the explanations of a 3 layers$\times$32 neurons ReLU neural network and perform the fooling 30 times for each subplot. We observe that centring the explanation in the attack loss definition is necessary to achieve the change in explanation shape. Alternatively, the explanation shifts upwards or downwards by essentially changing the mean of prediction. This observation was consistent across most of the models despite their complexity. 

\begin{figure}[!ht]
  \centering
  \includegraphics[width=0.85\columnwidth]{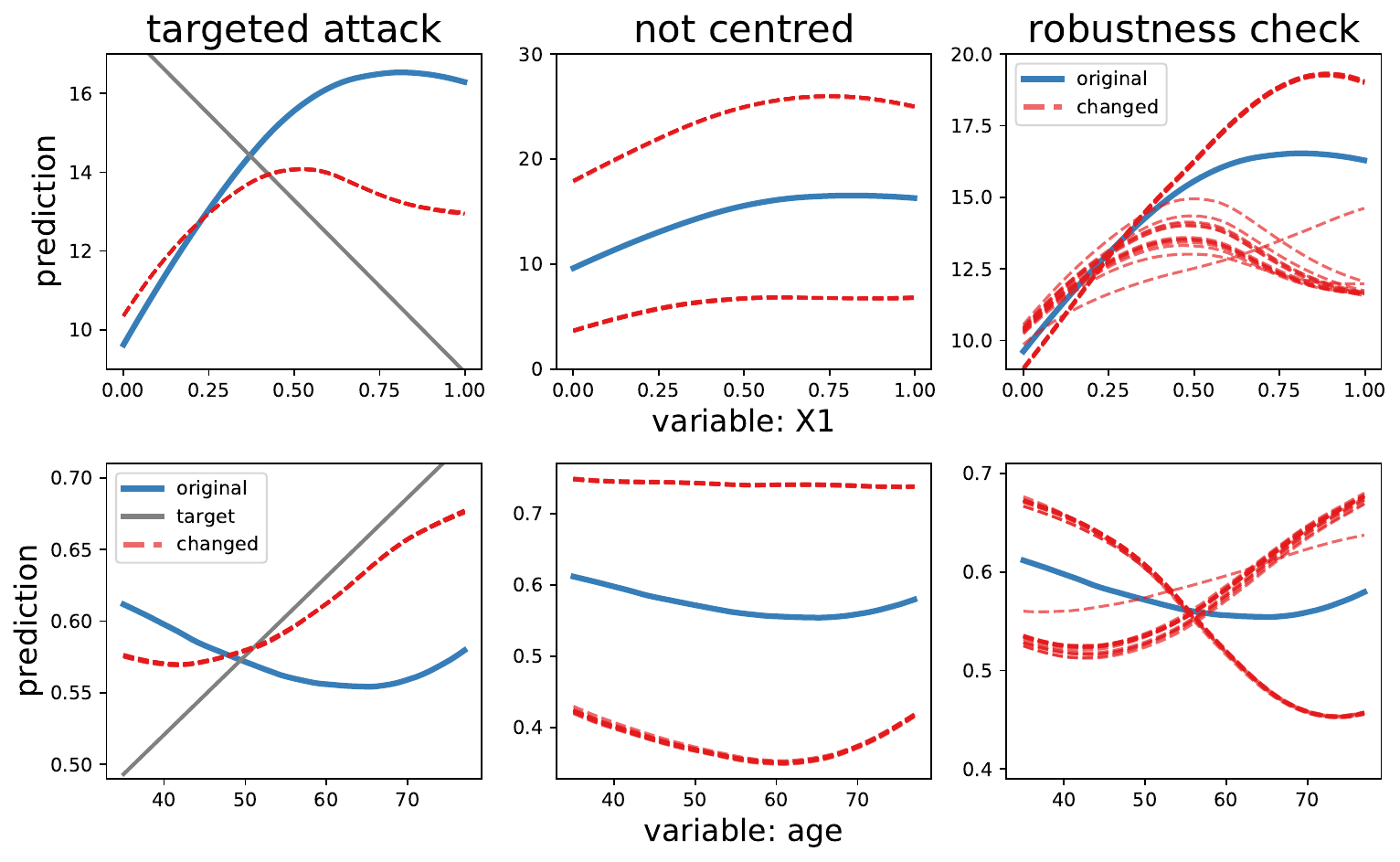}
  \caption{Fooling Partial Dependence of a 3$\times$32 neural network fitted to the \texttt{friedman} (\textbf{top} row) and \texttt{heart} (\textbf{bottom} row) datasets. We performed multiple randomly initiated gradient-based fooling algorithms on the explanations of variables $X_1$ and \texttt{age} respectively. 
  We observe that centring PD is beneficial because it stops the manipulated explanation from shifting.}\label{fig:add-experiments-centred}
\end{figure}

\begin{table}[!ht]
  \setlength{\tabcolsep}{4pt}    
  \centering
  \caption{Attack loss values of the robustness checks for PD of various ReLU neural networks. We add additional noise variables to the data before model fitting, e.g. \texttt{friedman}+2 denotes the referenced dataset with 2 additional variables sampled from the normal distribution. 
  We perform the fooling 30 times and report the $\text{mean} \pm \text{sd}$. We observe that the explanations' vulnerability greatly increases with task complexity.}
  \begin{tabular}{ccccccccc}
    \toprule
    \backslashbox{Task}{NN} & 1$\times$8 & 2$\times$8 & 3$\times$8 & 2$\times$32 & 3$\times$32 & 1$\times$128 & 3$\times$128\\   
    \midrule
    \texttt{friedman} & $25_{\pm 3}$ & $33_{\pm  0}$ & $75_{\pm  24}$ & $100_{\pm  32}$ & $98_{\pm  42}$ & $54_{\pm  15}$ & $97_{\pm  50}$\\ 
    \texttt{friedman}+1 & $31_{\pm  2}$ & $40_{\pm  4}$ & $50_{\pm  9}$ & $106_{\pm  40}$ & $115_{\pm  44}$ & $57_{\pm  15}$ & $114_{\pm  55}$\\ 
    \texttt{friedman}+2 & $34_{\pm  1}$ & $40_{\pm  10}$ & $50_{\pm  22}$ & $106_{\pm  52}$ & $115_{\pm  50}$ & $50_{\pm  15}$ & $137_{\pm  66}$\\ 
    \texttt{friedman}+4 & $46_{\pm  6}$ & $33_{\pm  0}$ & $83_{\pm  8}$ & $145_{\pm  31}$ & $163_{\pm  27}$ & $40_{\pm  5}$ & $140_{\pm  58}$\\ 
    \texttt{friedman}+8 & $71_{\pm  9}$ & $47_{\pm  3}$ & $89_{\pm  15}$ & $204_{\pm  25}$ & $176_{\pm  25}$ & $39_{\pm  6}$ & $156_{\pm  34}$\\ 
    \midrule
    \texttt{heart} & $11_{\pm  0}$ & $8_{\pm  1}$ & $10_{\pm  0}$ & $32_{\pm  3}$ & $41_{\pm  5}$ & $6_{\pm  1}$ & $134_{\pm  14}$\\ 
    \texttt{heart}+1 & $10_{\pm  1}$ & $17_{\pm  6}$ & $17_{\pm  2}$ & $44_{\pm  4}$ & $57_{\pm  13}$ & $6_{\pm  1}$ & $128_{\pm  8}$\\ 
    \texttt{heart}+2 & $13_{\pm  1}$ & $31_{\pm  13}$ & $17_{\pm  5}$ & $63_{\pm  4}$ & $79_{\pm  10}$ & $14_{\pm  2}$ & $218_{\pm  82}$\\ 
    \texttt{heart}+4 & $13_{\pm  1}$ & $21_{\pm  9}$ & $30_{\pm  17}$ & $113_{\pm  4}$ & $139_{\pm  60}$ & $29_{\pm  5}$ & $232_{\pm  36}$\\ 
    \texttt{heart}+8 & $16_{\pm  0}$ & $28_{\pm  18}$ & $43_{\pm  20}$ & $125_{\pm  49}$ & $227_{\pm  28}$ & $25_{\pm  8}$ & $311_{\pm  283}$\\ 
    \bottomrule
  \end{tabular}
  \label{tab:experimens-gradient}
\end{table}

Table \ref{tab:experimens-gradient} presents the impact of additional noise variables in data on the performed fooling. We observe that higher data dimensions favor vulnerable explanations (higher loss). The analogous results for targeted attack were consistent; however, showcased almost zero variance (partially observable in Figures \ref{fig:experiments-gradient} \& \ref{fig:add-experiments-centred}).

\paragraph{Adversarial Scenario.} Following the framework shown in Figure~\ref{fig:framework}, we consider three stakeholders apparent in explainable machine learning: developer, auditor, and prediction recipients. Let us assume that the model predicting a heart attack should not take into account a patient's \texttt{sex}; although, it might be a~valuable predictor. An auditor analyses the model using Partial Dependence; therefore, the developer supplies a poisoned dataset for this task. Figure \ref{fig:heart-sex} presents two possible outcomes of model audit: concealed and suspected, which are unequivocally bound to the explanation result and dataset. In first, the model is unchanged while the stated assumption of even dependence between \texttt{sex} is concealed (equal to about 0.5); thus, the prediction recipients become vulnerable. Additionally, we supply an alternative scenario where the developer wants to provide evidence of model unfairness to raise suspicion (dependence for class 0 equal to about 0.7).

\paragraph{Supportive Scenario.} In this work, we consider an equation of three variables: data, model, and explanation; thus, we poison the data to fool the explanation while the model remains unchanged. Figures \ref{fig:heart-age} and \ref{fig:heart-sex} showcase an exemplary data shift occurring in the dataset after the attack where changing only a few explanatory variables results in bending PD. We present a moderate change in data distribution to introduce a concept of analysing such relationships for explanatory purposes, e.g. the first result might suggest that resting blood pressure and maximum heart rate contribute to the explanation of \texttt{age}; the second result suggests how these variables contribute to the explanation of \texttt{sex}. We conclude that the data shift is worth exploring to analyse variable interactions in models. 
\section{Limitations and Future Work}\label{sec:limitations}

We find these results both alarming and informative yet proceed to discuss the limitations of the study. First is the assumption that, in an adversarial scenario, the auditor has no access to the original (unknown) data, e.g. in research or healthcare audit. While the detectability of fooling is worth analyzing, our work focuses not only on an adversarial manipulation of PD, as we sincerely hope such data poisoning is not occurring in practice. Even more, we aim to underline the crucial context of data distribution in the interpretation of explanations and introduce a new way of evaluating PD; black-box explanations by generalizing the methods with genetic-based optimization. 

Another limitation is the size of the used datasets. We have engaged with larger datasets during experiments but were turned off by a contradictory view that increasing the dataset size might be considered as exaggerating the results. PD clearly becomes more complex with increasing data dimensions; moreover, higher-dimensional space should entail more possible ways of manipulation, which is evident in the ablation study. We note that in practice, the explanations might require 100-1000 observations for the estimation (e.g. kernel SHAP, PDP), hence the size of the datasets in this study. Finally, we omit datasets like Adult and COMPAS because they mainly consist of categorical variables.

\begin{figure}
  \centering
  \includegraphics[width=0.49\textwidth]{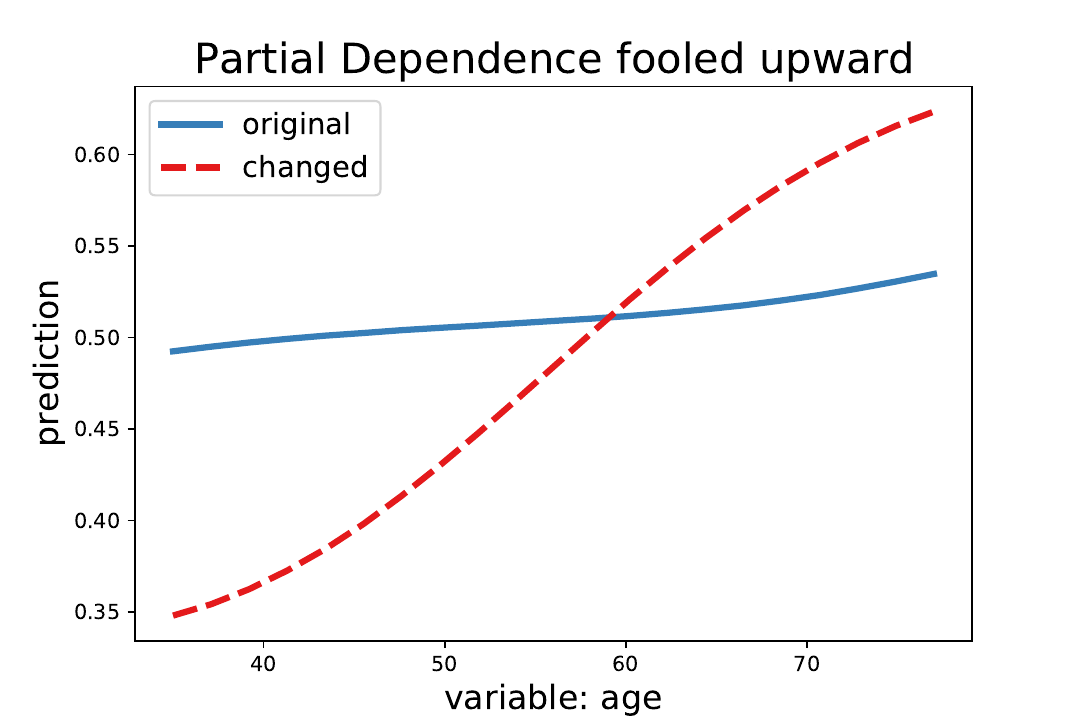}
  \includegraphics[width=0.32\textwidth]{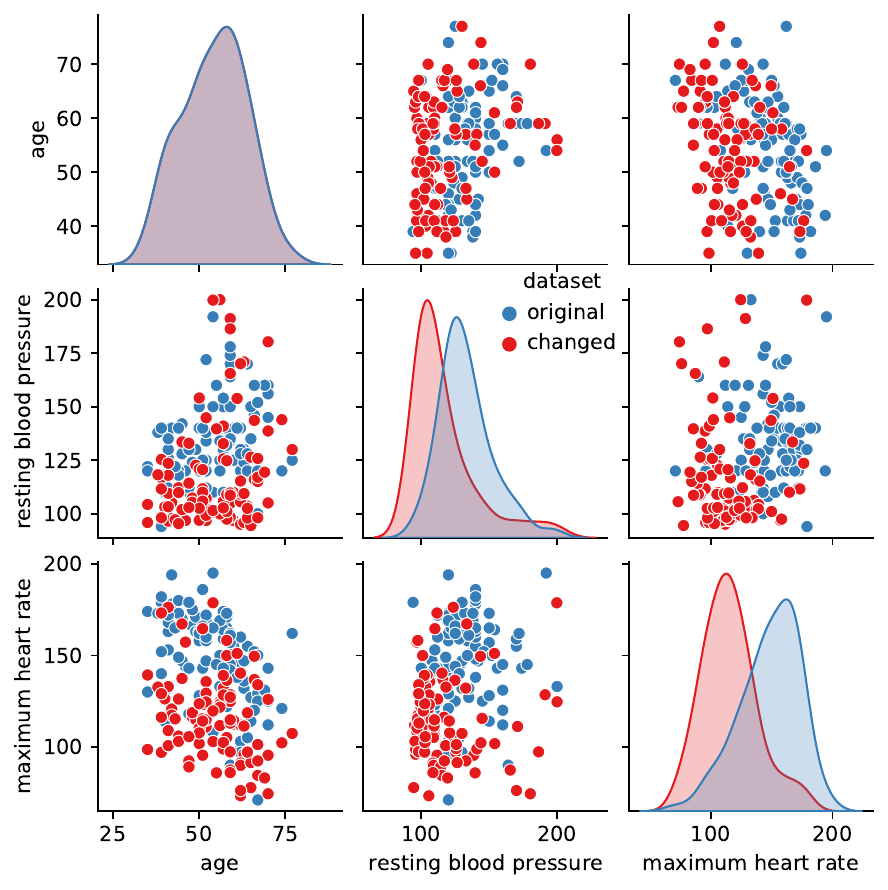}
  \includegraphics[width=0.49\textwidth]{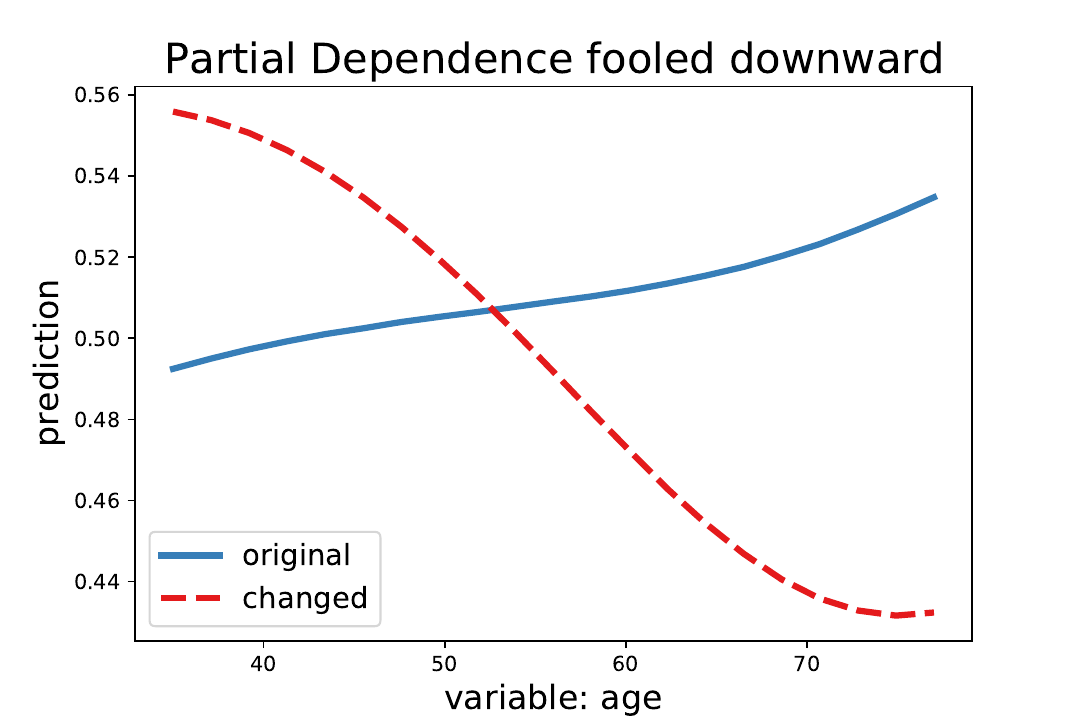}
  \includegraphics[width=0.32\textwidth]{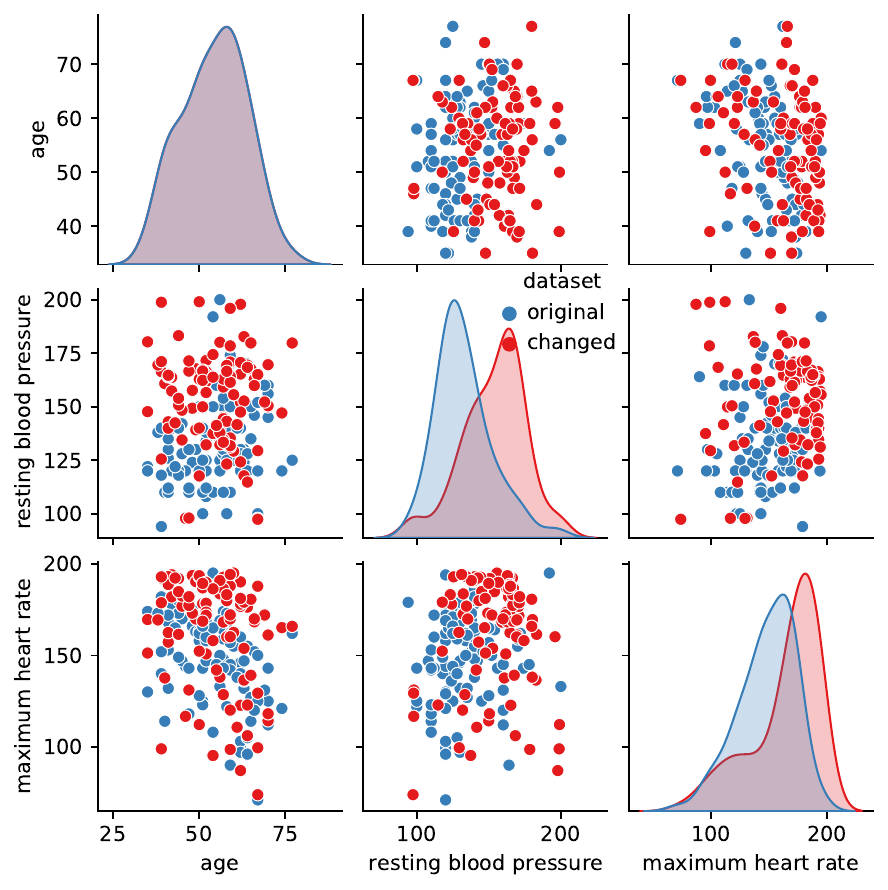}
  \caption{Partial Dependence of \texttt{age} in the SVM model prediction of a heart attack (class 0). \textbf{Left:} 
  Two manipulated explanations suggest an increasing or decreasing relationship between \texttt{age} and the predicted outcome depending on a desired outcome.
  \textbf{Right:} Distribution of the explained variable \texttt{age} and the two poisoned variables from the data, in which the remaining ten variables attributing to the explanation remain unchanged. The mean of the variables' Jensen-Shannon distance equals only 0.027 in the upward scenario and 0.021 in the downward scenario, which might seem like an insignificant change of the data distribution.}\label{fig:heart-age}
  
  \includegraphics[width=0.58\textwidth]{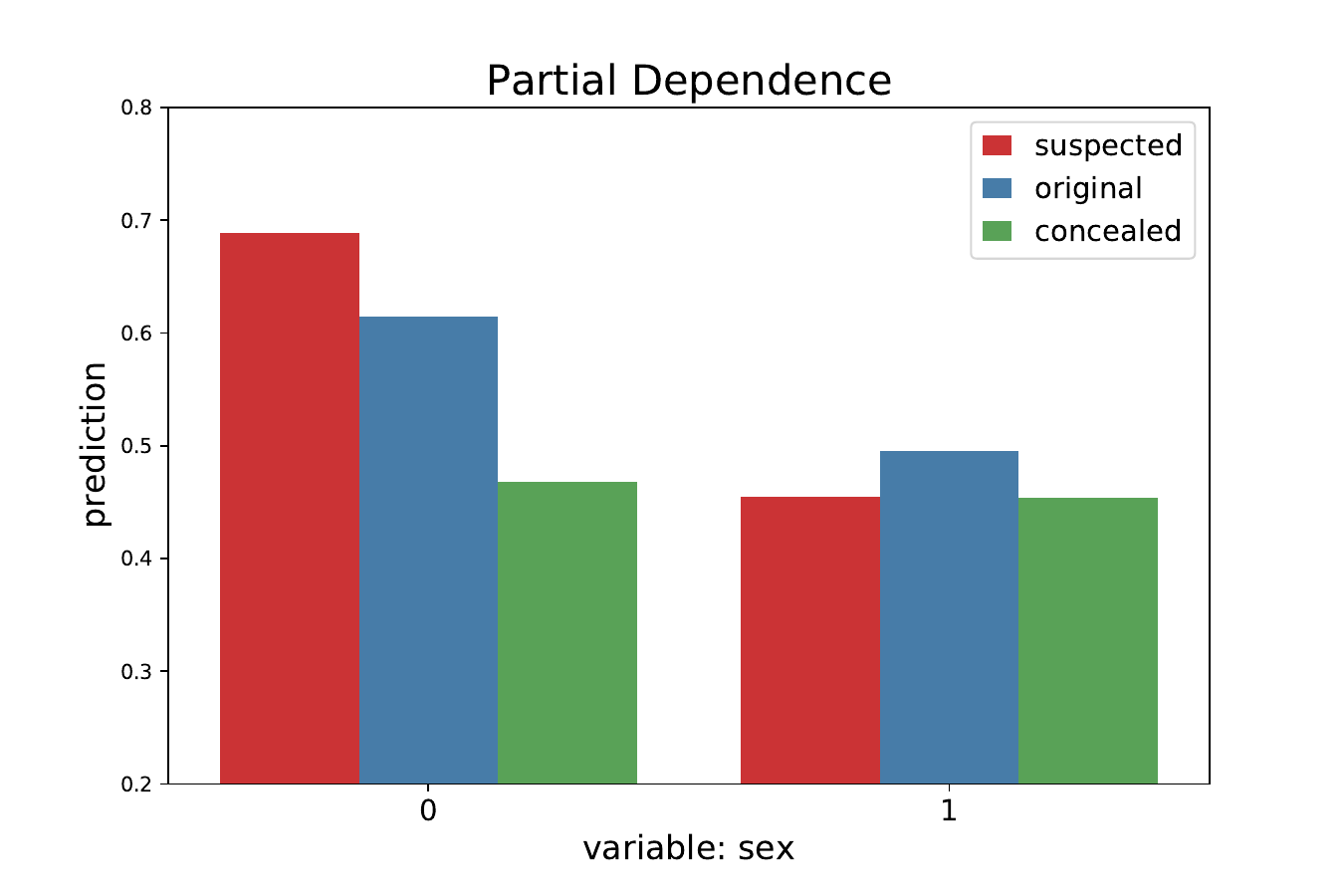}
  \includegraphics[width=0.39\textwidth]{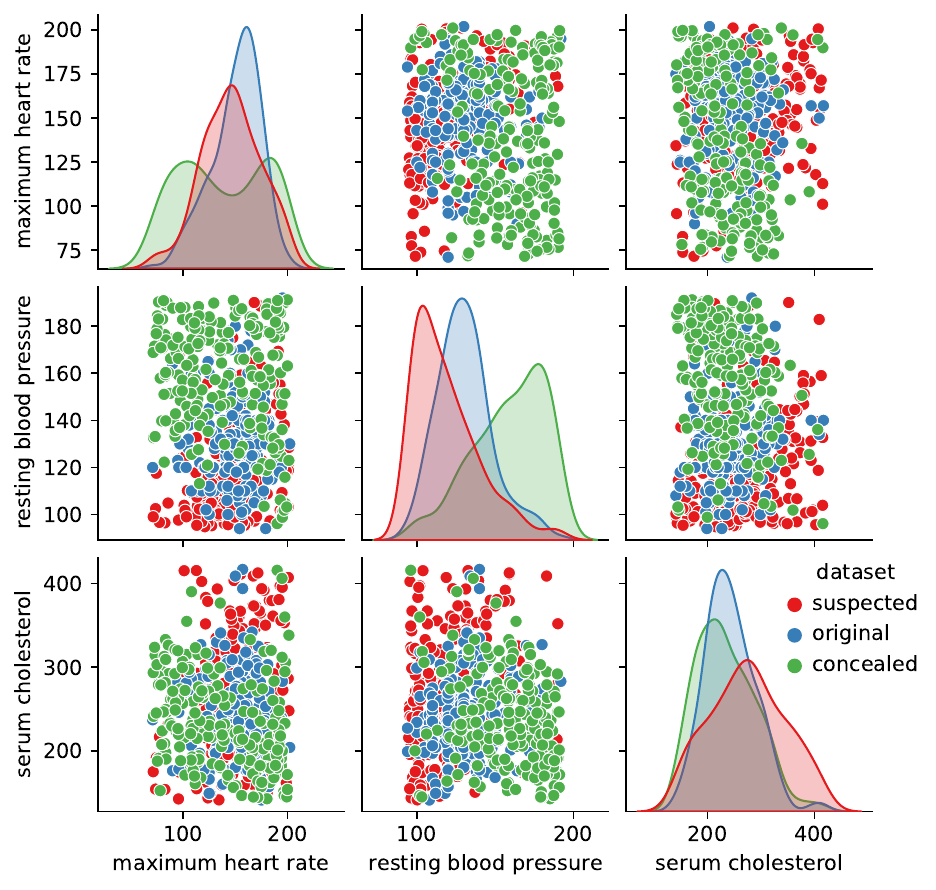}
  \caption{Partial Dependence of \texttt{sex} in the SVM model prediction of a heart attack (class~0). \textbf{Left:} Two manipulated explanations present a suspected or concealed variable contribution into the predicted outcome. \textbf{Right:} Distribution of the three poisoned variables from the data, in which \texttt{sex} and the remaining nine variables attributing to the explanation remain unchanged. The mean of the variables' Jensen-Shannon distance equals only 0.023 in the suspected scenario and 0.026 in the concealed scenario.}\label{fig:heart-sex}
\end{figure}

\paragraph{Future Work.} We foresee several directions for future work, e.g. evaluating the successor to PD -- Accumulated Local Effects (ALE) \citep{ale}; although the practical estimation of ALE presents challenges. Second, the attack loss may be enhanced by regularization, e.g. penalty for substantial change in data or mean of model's prediction, to achieve more meaningful fooling with less evidence. We focus in this work on univariate PD, but targeting bivariate PD can also be examined. Overall, the landscape of global-level, post-hoc model explanations is a broad domain, and the potential of a security breach in other methods, e.g. SHAP, should be further examined. Enhancements to the model-agnostic and explanation-agnostic genetic algorithm are thereby welcomed.

Another future direction would be to enhance the stability of PD. Rieger and Hansen \citep{defense-aggregation} present a defence strategy against the attack via data change~\citep{dombrowski-manipulated-explanations} by aggregating various explanations, which produces robust results without changing the model.

\section{Conclusion and Impact}\label{sec:conclusion}

We highlight that Partial Dependence can be maliciously altered, e.g. bent and shifted, with adversarial data perturbations. The introduced genetic-based algorithm allows evaluating explanations of any black-box model. Experimental results on various models and their sizes showcase the hidden debt of model complexity related to explainable machine learning. Explanations of low-variance models prove to be robust to the manipulation, while very complex models should not be explained with PD as they become vulnerable to change in reference data distribution. Robustness checks lead to varied modifications of the explanations depending on the setting, e.g. may propose two opposite PD, which is why it is advised to perform the checks multiple times.

This work investigates the vulnerability of global-level, post-hoc model explainability from the adversarial setting standpoint, which refers to the responsibility and security of the artificial intelligence use. Possible manipulation of PD leads to the conclusion that explanations used to explain black-box machine learning may be considered black-box themselves. These explainability methods are undeniably useful through implementations in various popular software. However, just as machine learning models cannot be developed without extensive testing and understanding of their behaviour, their explanations cannot be used without critical thinking. We recommend ensuring the reliability of the explanation results through the introduced methods, which can also be used to study models behaviour under the data shift. Code for this work is available at \url{https://github.com/MI2DataLab/fooling-partial-dependence}.

\subsubsection{Acknowledgements.} This work was financially supported
by the NCN Opus grant 2017/27/B/ST6/0130 and NCN Sonata Bis-9 grant 2019/34/E/ST6/00052.

%
%
%
\bibliographystyle{splncs04}
\bibliography{references}

\begin{thebibliography}{10}
\providecommand{\url}[1]{\texttt{#1}}
\providecommand{\urlprefix}{URL }
\providecommand{\doi}[1]{https://doi.org/#1}

\bibitem{adebayo-sanity-checks-saliency-maps}
Adebayo, J., Gilmer, J., Muelly, M., Goodfellow, I., Hardt, M., Kim, B.:
  {Sanity Checks for Saliency Maps}. In: NeurIPS (2018)

\bibitem{eval-4}
Adebayo, J., Muelly, M., Liccardi, I., Kim, B.: {Debugging Tests for Model
  Explanations}. In: NeurIPS (2020)

\bibitem{fairwashing}
Aivodji, U., Arai, H., Fortineau, O., Gambs, S., Hara, S., Tapp, A.:
  {Fairwashing: the risk of rationalization}. In: ICML (2019)

\bibitem{nn-software}
Alber, M., Lapuschkin, S., Seegerer, P., H{{\"a}}gele, M., Sch{{\"u}}tt, K.T.,
  et~al.: {iNNvestigate Neural Networks!} Journal of Machine Learning Research
  \textbf{20}(93), ~1--8 (2019)

\bibitem{ale}
Apley, D.W., Zhu, J.: {Visualizing the effects of predictor variables in black
  box supervised learning models}. Journal of the Royal Statistical Society:
  Series B (Statistical Methodology)  \textbf{82}(4),  1059--1086 (2020)

\bibitem{lrp}
Bach, S., Binder, A., Montavon, G., Klauschen, F., Müller, K.R., Samek, W.:
  {On Pixel-Wise Explanations for Non-Linear Classifier Decisions by Layer-Wise
  Relevance Propagation}. PLOS ONE  \textbf{10}(7),  1--46 (2015)

\bibitem{modelStudio}
Baniecki, H., Biecek, P.: {modelStudio: Interactive Studio with Explanations
  for ML Predictive Models}. Journal of Open Source Software  \textbf{4}(43),
  ~1798 (2019)

\bibitem{dalex-python}
Baniecki, H., Kretowicz, W., Piatyszek, P., Wisniewski, J., Biecek, P.: {dalex:
  Responsible Machine Learning with Interactive Explainability and Fairness in
  Python}. Journal of Machine Learning Research  \textbf{22}(214), ~1--7 (2021)

\bibitem{arrieta-responsible-ai}
{Barredo Arrieta}, A., Díaz-Rodríguez, N., {Del Ser}, J., Bennetot, A.,
  Tabik, S., et~al.: {Explainable Artificial Intelligence (XAI): Concepts,
  taxonomies, opportunities and challenges toward responsible AI}. Information
  Fusion  \textbf{58},  82--115 (2020)

\bibitem{eval-3}
Bhatt, U., Weller, A., Moura, J.M.F.: {Evaluating and Aggregating Feature-based
  Model Explanations}. In: IJCAI (2020)

\bibitem{dalex}
Biecek, P.: {DALEX: Explainers for Complex Predictive Models in R}. Journal of
  Machine Learning Research  \textbf{19}(84), ~1--5 (2018)

\bibitem{biecek-ema}
Biecek, P., Burzykowski, T.: {Explanatory Model Analysis}. Chapman and Hall/CRC
   (2021)

\bibitem{nn-defense-1}
Boopathy, A., Liu, S., Zhang, G., Liu, C., Chen, P.Y., Chang, S., Daniel, L.:
  {Proper Network Interpretability Helps Adversarial Robustness in
  Classification}. In: ICML (2020)

\bibitem{underspecification}
D'Amour, A., Heller, K., Moldovan, D., Adlam, B., Alipanahi, B., et~al.:
  {Underspecification Presents Challenges for Credibility in Modern Machine
  Learning}. arXiv preprint arXiv:2011.03395  (2020)

\bibitem{attack-model-fairness}
Dimanov, B., Bhatt, U., Jamnik, M., Weller, A.: {You Shouldn't Trust Me:
  Learning Models Which Conceal Unfairness From Multiple Explanation Methods.}
  In: AAAI SafeAI (2020)

\bibitem{dombrowski-manipulated-explanations}
Dombrowski, A.K., Alber, M., Anders, C., Ackermann, M., M\"{u}ller, K.R.,
  Kessel, P.: {Explanations can be manipulated and geometry is to blame}. In:
  NeurIPS (2019)

\bibitem{UCI}
Dua, D., Graff, C.: {UCI Machine Learning Repository} (2017),
  \url{https://www.kaggle.com/ronitf/heart-disease-uci/version/1}

\bibitem{5-evolutionary-algorithms}
Elbeltagi, E., Hegazy, T., Grierson, D.: {Comparison among five
  evolutionary-based optimization algorithms}. Advanced Engineering Informatics
   \textbf{19}(1),  43--53 (2005)

\bibitem{friedman-gbm-pdp}
Friedman, J.H.: {Greedy Function Approximation: A Gradient Boosting Machine}.
  Annals of Statistics  \textbf{29}(5),  1189--1232 (2001)

\bibitem{faking-fairness}
Fukuchi, K., Hara, S., Maehara, T.: {Faking Fairness via Stealthily Biased
  Sampling}. In: AAAI (2020)

\bibitem{ghorbani-fragile-nn-interpretability}
Ghorbani, A., Abid, A., Zou, J.: {Interpretation of Neural Networks Is
  Fragile}. In: AAAI (2019)

\bibitem{responsible-ml}
Gill, N., Hall, P., Montgomery, K., Schmidt, N.: {A Responsible Machine
  Learning Workflow with Focus on Interpretable Models, Post-hoc Explanation,
  and Discrimination Testing}. Information  \textbf{11}(3), ~137 (2020)

\bibitem{goldstein-ice}
Goldstein, A., Kapelner, A., Bleich, J., Pitkin, E.: {Peeking Inside the Black
  Box: Visualizing Statistical Learning With Plots of Individual Conditional
  Expectation}. Journal of Computational and Graphical Statistics
  \textbf{24}(1),  44--65 (2015)

\bibitem{pdp}
Greenwell, B.M.: {pdp: An R Package for Constructing Partial Dependence Plots}.
  The R Journal  \textbf{9}(1),  421--436 (2017)

\bibitem{heo-fooling-nn-manipulation}
Heo, J., Joo, S., Moon, T.: {Fooling Neural Network Interpretations via
  Adversarial Model Manipulation}. In: NeurIPS (2019)

\bibitem{hooker2007}
Hooker, G.: {Generalized Functional ANOVA Diagnostics for High-Dimensional
  Functions of Dependent Variables}. Journal of Computational and Graphical
  Statistics  \textbf{16}(3),  709--732 (2007)

\bibitem{eval-1}
Hooker, S., Erhan, D., Kindermans, P.J., Kim, B.: {A Benchmark for
  Interpretability Methods in Deep Neural Networks}. In: NeurIPS (2019)

\bibitem{janzing2020feature}
Janzing, D., Minorics, L., Bl{\"o}baum, P.: {Feature relevance quantification
  in explainable AI: A causal problem}. In: AISTATS (2020)

\bibitem{explanation-quality-vs-model-accuracy}
Jia, Y., Frank, E., Pfahringer, B., Bifet, A., Lim, N.: {Studying and
  Exploiting the Relationship Between Model Accuracy and Explanation Quality}.
  In: ECML PKDD (2021)

\bibitem{kindermans-unreliable-saliency-maps}
Kindermans, P.J., Hooker, S., Adebayo, J., Alber, M., Sch{\"u}tt, K.T.,
  D{\"a}hne, S., Erhan, D., Kim, B.: {The (Un)reliability of Saliency Methods}.
  In: {Explainable AI: Interpreting, Explaining and Visualizing Deep Learning},
  pp. 267--280. Springer (2019)

\bibitem{adam}
Kingma, D.P., Ba, J.: {Adam: A Method for Stochastic Optimization}. In: ICLR
  (2015)

\bibitem{misleading}
Lakkaraju, H., Bastani, O.: {"How Do I Fool You?": Manipulating User Trust via
  Misleading Black Box Explanations}. In: AIES (2020)

\bibitem{muse}
Lakkaraju, H., Kamar, E., Caruana, R., Leskovec, J.: {Faithful and Customizable
  Explanations of Black Box Models}. In: AIES (2019)

\bibitem{deeplearning}
LeCun, Y., Bengio, Y., Hinton, G.: {Deep learning}. Nature  \textbf{521}(7553),
   436--444 (2015)

\bibitem{Lipton2018}
Lipton, Z.C.: {The Mythos of Model Interpretability}. Queue  \textbf{16}(3),
  31--–57 (2018)

\bibitem{shap}
Lundberg, S.M., Lee, S.I.: {A Unified Approach to Interpreting Model
  Predictions}. In: NeurIPS (2017)

\bibitem{adversarial-robustness}
Mangla, P., Singh, V., Balasubramanian, V.N.: {On Saliency Maps and Adversarial
  Robustness}. In: ECML PKDD (2020)

\bibitem{Miller2019}
Miller, T.: {Explanation in artificial intelligence: Insights from the social
  sciences}. Artificial Intelligence  \textbf{267},  1--38 (2019)

\bibitem{iml}
Molnar, C., Casalicchio, G., Bischl, B.: {iml: An R package for Interpretable
  Machine Learning}. Journal of Open Source Software  \textbf{3}(26), ~786
  (2018)

\bibitem{lime}
Ribeiro, M.T., Singh, S., Guestrin, C.: {“Why Should I Trust You?”:
  Explaining the Predictions of Any Classifier}. In: KDD (2016)

\bibitem{defense-aggregation}
Rieger, L., Hansen, L.K.: {A simple defense against adversarial attacks on
  heatmap explanations}. In: ICML WHI (2020)

\bibitem{Rudin2019}
Rudin, C.: {Stop Explaining Black Box Machine Learning Models for High Stakes
  Decisions and Use Interpretable Models Instead}. Nature Machine Intelligence
  \textbf{1},  206--215 (2019)

\bibitem{gradcam}
Selvaraju, R.R., Cogswell, M., Das, A., Vedantam, R., Parikh, D., Batra, D.:
  {Grad-CAM: Visual Explanations from Deep Networks via Gradient-Based
  Localization}. International Journal of Computer Vision  \textbf{128}(2),
  336--359 (2020)

\bibitem{input-gradient}
Shrikumar, A., Greenside, P., Kundaje, A.: {Learning Important Features Through
  Propagating Activation Differences}. In: ICML (2017)

\bibitem{saliency1}
Simonyan, K., Vedaldi, A., Zisserman, A.: {Deep Inside Convolutional Networks:
  Visualising Image Classification Models and Saliency Maps}. In: ICLR (2014)

\bibitem{slack-fooling-lime-shap}
Slack, D., Hilgard, S., Jia, E., Singh, S., Lakkaraju, H.: {Fooling LIME and
  SHAP: Adversarial Attacks on Post Hoc Explanation Methods}. In: AIES (2020)

\bibitem{slack2021counterfactual}
Slack, D., Hilgard, S., Lakkaraju, H., Singh, S.: {Counterfactual Explanations
  Can Be Manipulated}. In: NeurIPS (2021)

\bibitem{poisoning-fairness}
Solans, D., Biggio, B., Castillo, C.: {Poisoning Attacks on Algorithmic
  Fairness}. In: ECML PKDD (2020)

\bibitem{integrated-gradient}
Sundararajan, M., Taly, A., Yan, Q.: {Axiomatic Attribution for Deep Networks}.
  In: ICML (2017)

\bibitem{nn-defense-2}
Wang, Z., Wang, H., Ramkumar, S., Mardziel, P., Fredrikson, M., Datta, A.:
  {Smoothed Geometry for Robust Attribution}. In: NeurIPS (2020)

\bibitem{eval-2}
Warnecke, A., Arp, D., Wressnegger, C., Rieck, K.: {Evaluating Explanation
  Methods for Deep Learning in Security}. In: IEEE EuroS\&P (2020)

\bibitem{real-ga}
Wright, A.H.: {Genetic Algorithms for Real Parameter Optimization}. Foundations
  of Genetic Algorithms  \textbf{1},  205--218 (1991)

\bibitem{zhang-interpretable-dl-under-fire}
Zhang, X., Wang, N., Shen, H., Ji, S., Luo, X., Wang, T.: {Interpretable Deep
  Learning under Fire}. In: USENIX Security (2020)

\bibitem{zhao2019causal}
Zhao, Q., Hastie, T.: {Causal Interpretations of Black-Box Models}. Journal of
  Business \& Economic Statistics  \textbf{39}(1),  272--281 (2019)

\end{thebibliography}

\appendix

\section{Genetic-based Algorithm}\label{appendix:genetic}

Attacks on PD in both strategies include a similar Algorithm \ref{alg:genetic-based-algorithm}. The main idea is defining an individual as an instance of the dataset, iteratively perturb its values to achieve the desired explanation target, or perform the robustness check to observe the change. These individuals are initialized with a value of original dataset $X'$ to form a population $P$. Subsequently, the initialization ends with mutating $P$ using a higher-than-default variance of perturbations. Then, in each iteration, they are randomly crossed, mutated, evaluated with the loss function, and selected based on the evaluation. The algorithm stops after a defined number of repetitions, and the best individual, with its corresponding explanation, is the result. The initialized population moves to the crossover phase.

\begin{algorithm}
 \KwData{$f,\;X',\;g,\;c,\;T,\;C $}
 \KwResult{$g_c(X),\;X$}
 initialize $P$ \;
 \While{$iteration < max\_iterations$}{
  crossover phase (Algorithm \ref{alg:crossover}) \;
  mutation phase (Algorithm \ref{alg:mutation}) \;
  evaluation phase (Algorithm \ref{alg:evaluation}) \;
  \If{$iteration < max\_iterations - 1$}{
    selection phase (Algorithm \ref{alg:selection}) \;
  }
 }
 $X \gets \operatorname*{argmin}_{x\in P}{\mathcal{L}(x)} $\;
 \caption{Data poisoning using a genetic-based algorithm.}
 \label{alg:genetic-based-algorithm}
\end{algorithm}

The crossover presented in Algorithm \ref{alg:crossover} swaps columns between parent individuals to produce new ones. The proportion of the population which becomes parents is parameterized by $crossover\_ratio$, and the parent pairs are randomly sampled without replacement from the subset $P_{crossover\_ratio}$. For each pair, the set of variable columns (full dataset) to swap is randomly selected, and becomes a~newly created individual $q$. Such constructed childs $Q$ are added to the population. The enlarged population moves to the mutation phase.

\begin{algorithm}
 \KwData{$P,\;crossover\_ratio$}
 \KwResult{$P$}
 $R \gets P_{crossover\_ratio}$\;
 $Q \gets \{\} $\;
 \While{there are individuals in $R$}{
    $n,\;m \gets$ sample a pair of individuals without replacement from $R$ \;
    $q \gets$ create a new individual from the randomly selected columns of $n$ and $m$ \;
    $Q \gets Q \cup q $ \; 
 }
 $ P \gets P \cup Q$ \;
 \caption{Crossover phase.}
 \label{alg:crossover} 
\end{algorithm}

The mutation presented in Algorithm \ref{alg:mutation} adds Gaussian noise to the individuals (datasets) using the variables' standard deviations $std(X')$, which results in a changed population $P$. These standard deviations are scaled by the $std\_ratio$ parameter to lower the variance of noise. There is a possibility to constraint the changes in the datasets only to the original range of variable values. Then, the potential incorrect values that might occur, are substituted with new ones from the uniform distribution of the range between the original dataset value and the boundaries. It is also practicable to treat chosen elements of the dataset as constant. The mutated population moves to the evaluation phase.

\begin{algorithm}
 \KwData{$P,\;X',\;C,\;std\_ratio,\;mutation\_with\_constraints$}
 \KwResult{$P$}
 \For{each individual $m \in P$}{
    \tcp{Gaussian noise with mean $0$}
    $\theta \gets noise(std(X') \boldsymbol{\cdot} std\_ratio)$ \;
    $mask \gets create\_mask(X',\;C)$\;
    $m \gets m + \theta \boldsymbol{\cdot} mask $\;
    \If{mutation\_with\_constraints}{
        \For{each column $v \in m$}{
            find values which are out of the original range $\left[min\left(v\right), max\left(v\right)\right]$ \;
            sample new values from the uniform distribution $U\left[min\left(v\right),\; v\_i\right]$ or $U\left[v\_i,\;max\left(v\right)\right]$ \;
            substitute the out-of-range values \;
        }
    }
 }
 \caption{Mutation phase.}
 \label{alg:mutation}
\end{algorithm}

The evaluation presented in Algorithm \ref{alg:evaluation} uses the Attack loss function; thus, depends on the strategy $s$. For the robustness check $r$ we use the original dataset $X'$ to calculate the loss $\mathcal{L}^{g,\;r}$, while for the targeted attack $t$ we require $T$ in $\mathcal{L}^{g,\;t}$. Genetic algorithms usually maximize the fitness function, but we decided to minimize the loss function so that both considered algorithms are similar. Algorithm \ref{alg:evaluation} returns loss values $l$ for each individual which are passed to the selection phase.

\begin{algorithm}
 \KwData{$P,\;X',\;g,\;c,\;s,\;T$}
 \KwResult{$l$ \, \tcp{loss calculated for each individual from the population}} 
 $l \gets \{\} $\;
 \For{each individual $m \in P$}{
    \tcp{$\mathcal{L}^{g,\;r}$ uses $X'$, while $\mathcal{L}^{g,\;t}$ uses $T$}
    $l_m \gets \mathcal{L}^{g,\;s}(m)$ 
 }    
 \caption{Evaluation phase.}
 \label{alg:evaluation}
\end{algorithm}

The selection presented in Algorithm \ref{alg:selection} uses the rank selection algorithm to reduce the number of individuals to the $pop\_count$ starting number and ensure attack convergence. Rank selection uses the probability of survival of each individual, which depends on their ranking based on the corresponding loss values $l$. We added fundamental elitism to the selection algorithm, meaning that in each iteration, we guarantee several best individuals to remain into the next population. This addition ensures that the genetic-based attack's solution quality will not decrease from one iteration to the next. The cycle continues until \textit{$max\_iter$} iterations are reached, and the best individual is selected.

\begin{algorithm}
 \KwData{$P,\;l,\;pop\_count,\;elitism\_count$}
 \KwResult{$P$}
 $P' \gets P$ ordered by the values of $l$ \;
 $E \gets elitism\_count$ best individuals from $P$ \;
 $prob \gets rank(P')$ \;
 $P \gets sample(P',\;prob,\;pop\_count)$ \;
 $P \gets P \cup E$ \;
 \caption{Selection phase.}
 \label{alg:selection}
\end{algorithm}

\clearpage 

\section{Gradient-based Algorithm}\label{appendix:gradient}

Gradient-based algorithm uses gradient of the attack loss, which can be further enhanced by optimizers such as Adam. The implementation relies on automatic differentiation, which is easily accessible using nowadays tools like PyTorch or TensorFlow. Next, we provide Equations \ref{eq:eq1} and \ref{eq:eq2} of the attack loss derivatives used in the algorithm. When calculating the derivative of $\mathcal{PD}_c$, we assume that the explained variable $c$ is constant; thus, we denote the gradient as $\nabla_{X_{-c}}$.

\begin{lemma}[Derivative of $\mathcal{L}^{\mathcal{PD},\;t}$ and $\mathcal{L}^{\mathcal{PD},\;r}$]\label{derivative-loss-first-term-pd}
Let $f:\mathbb{R}^{N \times P} \longrightarrow \mathbb{R}^{N}$ represents the differentiable function that is explained by $\mathcal{PD}$. Let $Z$ be the set of points used to calculate $\mathcal{PD}$. Let $T:Z\rightarrow \mathbb{R}$. Finally, let $X'\in \mathbb{R}^{N \times P}$ be the original dataset. Then
\begin{equation} \label{eq:eq1}
\begin{split}
\nabla_{X_{-c}} & \mathcal{L}^{\mathcal{PD},\;t}(X) = \frac{2}{N|Z|}\sum_{z\in Z}\nabla_{X_{-c}} f(X^{c|=z}) \boldsymbol{\cdot} \left(\mathcal{PD}_c(X, z) - T(z)\right), \\
\nabla_{X_{-c}} & \mathcal{L}^{\mathcal{PD},\;r}(X) = -\frac{2}{N|Z|}\sum_{z\in Z}\nabla_{X_{-c}} f(X^{c|=z}) \boldsymbol{\cdot} \left(\mathcal{PD}_c(X, z) - \mathcal{PD}_c(X', z)\right).
\end{split}
\end{equation}
\end{lemma}

\begin{proof}
We derive the formula for the targeted attack, while a formula for the robustness check is derived by analogy. We calculate the derivative with respect to a particular value $X_{i,j}$ in the dataset. Note that we want to leave column $c$ intact, so $j\neq c$. $T(z)$ is independent of $X_{i,j}$, so it is dropped during the differentation. We have

\begin{equation*}
    \frac{\partial \mathcal{L}^{\mathcal{PD},\;t}(X)}{\partial X_{i,j}} = \frac{\partial}{\partial X_{i,j}} \frac{1}{|Z|}\sum_{z\in Z}\left(\mathcal{PD}_c(X, z)-T(z)\right)^2 =
\end{equation*}
\begin{equation*}
    \frac{2}{|Z|}\sum_{z\in Z}\left(\mathcal{PD}_c(X, z) - T(z)\right)\frac{\partial}{\partial X_{i,j}}\frac{1}{N}\sum_{k=1}^N f(X_k^{c|=z}) =
\end{equation*}
\begin{equation*}
    \frac{2}{N|Z|}\sum_{z\in Z}\left(\mathcal{PD}_c(X, z) - T(z)\right)\frac{\partial}{\partial X_{i,j}}f(X_i^{c|=z}).
\end{equation*}
\end{proof}

\begin{lemma}[Derivative of $\mathcal{L}^{\mathcal{\overline{PD}},\;r}$]\label{derivative-loss-first-term-pdp-centred}
Let $f:\mathbb{R}^{N \times P} \longrightarrow \mathbb{R}^{N}$ represents the differentiable function that is explained by $\mathcal{PD}$. Let $Z$ be the set of points used to calculate $\mathcal{PD}$. Let $\mathcal{\overline{PD}}$ denote the centred $\mathcal{PD}$, which is obtained by substracting the mean in each point. Finally, let $X'\in \mathbb{R}^{N \times P}$ be the original dataset. Then
\begin{equation} \label{eq:eq2}
\begin{split}
\nabla_{X_{-c}} & \mathcal{L}^{\mathcal{\overline{PD}},\;r}(X) = - \frac{2}{N|Z|}\sum_{z\in Z}\left(\nabla_{X_{-c}} f(X^{c|=z})-\frac{\sum_{z'\in Z}\nabla_{X_{-c}} f(X^{c|=z'})}{|Z|}\right) \\ & \boldsymbol{\cdot} \left(\mathcal{\overline{PD}}_c(X, z) - \mathcal{\overline{PD}}_c(X', z)\right).
\end{split}
\end{equation}

\end{lemma}

\begin{proof}\label{proof-derivative-loss-first-term-pdp-centred}
We calculate the derivative with respect to a particular value $X_{i,j}$ in the dataset. Please note, that we want to leave column $c$ intact, so $j\neq c$. $\mathcal{\overline{PD}}_c(X',\;z)$ is independent of $X_{i,j}$ so it is dropped during the differentation. We have

\begin{equation*}
    \frac{\partial \mathcal{L}^{\mathcal{\overline{PD}},\;r}(X)}{\partial X_{i,j}} = - \frac{\partial}{\partial X_{i,j}} \frac{1}{|Z|}\sum_{z\in Z}\left(\mathcal{\overline{PD}}_c(X,\;z) - \mathcal{\overline{PD}}_c(X',\;z) \right)^2 =
\end{equation*}
\begin{equation*}
    - \frac{2}{|Z|}\sum_{z\in Z}\left(\mathcal{\overline{PD}}_c(X,\;z) - \mathcal{\overline{PD}}_c(X',\;z)\right)
\end{equation*}
\begin{equation*}   
    \boldsymbol{\cdot}  \left(\frac{1}{N}\sum_{k=1}^N\frac{\partial}{\partial X_{i,j}}f(X_k^{c|=z}) - \frac{1}{|Z|}\sum_{z'\in Z}\frac{1}{N}\sum_{k=1}^N \frac{\partial}{\partial X_{i,j}}f(X_k^{c|=z'})\right) =
\end{equation*}
\begin{equation*}
    - \frac{2}{N|Z|}\sum_{z\in Z}\left(\mathcal{\overline{PD}}_c(X,\;z) - \mathcal{\overline{PD}}_c(X',\;z)\right)
\end{equation*}
\begin{equation*}
    \boldsymbol{\cdot} \left(\frac{\partial}{\partial X_{i,j}}f(X_i^{c|=z}) - \frac{1}{|Z|}\sum_{z'\in Z}\frac{\partial}{\partial X_{i,j}}f(X_i^{c|=z'})\right)\text{.}
\end{equation*}
\end{proof}

Algorithm \ref{alg:gradient-based-algorithm} shows the straightforward gradient-based idea of attack loss optimization. Additional arguments passed to this algorithm depend on the chosen optimizer method. Learning rate is crucial in this case. It is denoted as $\eta$ and it controls a step size in each iteration. Greater learning rate usually results in faster algorithm convergence but may lead to worse solutions, wherein the worst-case scenario, the algorithm might not converge at all. On the other hand, a lower learning rate may result in too slow convergence and not optimal result. We use learning rate equal to $0.01$ by default. It is simple to set some of the columns in the dataset, or even particular values, as constants. It means that they are not considered arguments of the loss function; thus, they do not change during the attack. It is done simply by setting to $0$ partial derivatives corresponding to chosen values in the gradient of loss function. Note, that explained column $c$ is constant, thus its partial derivative is always set to $0$.

\begin{algorithm}
 \KwData{$f,\;X',\;g,\;c,\;T,\;C$}
 \KwResult{$g_c(X),\;X$}
 $X \gets X' + noise()$\;
 \While{iteration < max\_iterations}{
    \tcp{$\mathcal{L}^{g,\;r}$ uses $X'$, while $\mathcal{L}^{g,\;t}$ uses $T$}
    $l \gets \mathcal{L}^{g,\;s}(X,\;f)$\;
    $l \gets set\_to\_0(l,\;C)$\;
    $X \gets X - \eta \boldsymbol{\cdot} Adam(l)$\;
 }
 \caption{Data poisoning using a gradient-based algorithm.}
 \label{alg:gradient-based-algorithm}
\end{algorithm}

\end{document}